\documentclass{llncs}

\usepackage{hyperref}
 \usepackage{latexsym}
\usepackage{amsmath}
 \usepackage{amssymb}
\usepackage{xspace}
\usepackage{tabularx}
\usepackage{caption}
\usepackage{url}
\usepackage{color}

\usepackage{graphicx}
\newtheorem{Def}[theorem]{Definition}
\newtheorem{observation}[theorem]{Observation}

\newcommand{\R}{\ensuremath{\mathbb{R}}}

\newcommand{\alert}[1]{\ifbool{havealerts}{\marginnote{\textcolor{red}{#1}}}{}}

\renewcommand{\epsilon}[0]{\varepsilon}

\newcommand{\shrinkspace}[1]{\ifbool{fullversion}{}{\vspace{#1}}}
\newcommand{\beq}{\begin{eqnarray}}
\newcommand{\eeq}{\end{eqnarray}}

\def\eps{{\epsilon}}
\renewcommand{\(}{\left(}
\renewcommand{\)}{\right)}

\newcommand{\D}[1]{{\sc FRW$_{#1, \delta}$}}
\renewcommand{\O}[1]{{\sc Opt-FRW$_{#1, \delta}$}}
\newcommand{\C}[1]{{\sc AFRW$_{#1, \delta}$}}
\newcommand{\A}[1]{{\sc AOpt-FRW$_{#1, \delta}$}}
\newcommand{\E}{\ensuremath{\mathbb{E}}}

\numberwithin{equation}{section}

 \newtheorem{Claim}[theorem]{Claim}

\newenvironment{Proof}{\medbreak \noindent {\bf
Proof:~}}{\unskip\nobreak\hfill\hskip 2em \bull \par\medbreak}
\def\bull{\vrule height .9ex width .8ex depth -.1ex }

\begin{document}

\title{Fractal structures in Adversarial Prediction}
\titlerunning{Fractal structures in Adversarial Prediction}  

\author{Rina Panigrahy\inst{1}
\and
Preyas Popat\inst{2}
}

\institute{
Microsoft Corp. \\ Mountain View, CA \\ rina@microsoft.com
\and
New York University and University of Chicago \\ New York, NY and Chicago, IL \\ popat@cs.nyu.edu
}

\maketitle

\begin{abstract}
Fractals are self-similar recursive structures that
 have been used in modeling several real world processes. In this work we study how ``fractal-like'' processes arise in a prediction game where an adversary is 
generating a sequence of bits and an algorithm is trying to predict them. We will see that under a certain formalization of the predictive payoff for the algorithm 
it is most optimal for the adversary to produce a fractal-like sequence to minimize the algorithm's ability to predict.   
Indeed it has been suggested before that financial markets exhibit a fractal-like behavior \cite{frost1998elliott,mandelbrot2005inescapable}. 
We prove that a fractal-like distribution arises naturally out of an optimization from the adversary's perspective.

In addition, we give optimal trade-offs between predictability and expected deviation (i.e. sum of bits) for our formalization of predictive payoff.  This result is motivated by the observation
that several time series data exhibit higher deviations than expected for a completely random walk.
\end{abstract}





\section{Introduction}

Consider an adversary who is producing a sequence of bits (each bit is $+1$ or $-1$) and an algorithm having seen a certain number of bits is interested in predicting the next $x$ bits. 
Say the algorithm gets a payoff of $1$ for every bit that it predicts correctly and $-1$ for every bit where it is wrong.  
This is like an idealized stock market where each day the price changes by $+1$ or $-1$ percent and the algorithm is required to make a bet on the daily direction. 
We ask what is the most adversarial distribution on sequence of bits so as to minimize the algorithm's payoff. Clearly the uniform distribution where every 
bit is chosen independently and uniformly at random is the most adversarial, since the expected payoff of any algorithm is always exactly $0$.  

Given a sequence $s$ of bits,
let $h(s)$ be the sum of the bits in $s$ i.e. the \emph{height} of the sequence when plotted cumulatively.  We will refer to the magnitude of height as deviation.
For $s \in \{-1, 1\}^T$ chosen uniformly at random the typical deviation $s$ is $\Theta(\sqrt{T})$. 

The question we study here is: what is the most adversarial distribution on sequences if the distribution is required to be heavy-tailed, 
say the typical height should be $k \sqrt T$ where $k > 1$. Indeed it has been observed in several studies that the distribution of financial time 
series is heavy-tailed \cite{bradley2003financial,rachev2005fat}.   A natural heavy-tailed distribution is to pick a random string conditioned on its height being at least $k \sqrt T$. 
This is essentially the highest entropy distribution with the property that the typical height is around $k \sqrt T$. 
However the highest entropy distribution is not the least predictable. Indeed for large $k$, it tends to rise/drop rather linearly to its final height. 
Thus by observing the initial segment of bits, the algorithm can easily infer the direction of the remaining bits to get a large payoff. 

One distribution that has been suggested for financial markets is the Fractional Brownian Motion (FBM) \cite{mandelbrot1968fractional,nualart2006fractional} which is a generalization 
of the Brownian motion.  For our purposes, the Brownian motion can be thought of as a continuous variant of the uniform distribution on bits.  FBM is characterized by a single 
parameter $H$ which is called the Hurst parameter, and the typical height achieved by sequences drawn from FBM$(H)$ is around $T^H$.  For $H > 1/2$, the increments of FBM
are positively correlated while the case $H = 1/2$ corresponds to Brownian motion.

To make our question precise we introduce a measure of unpredictability for a distribution which is motivated by the notion that the expected payoff of an algorithm on an interval $I$
having observed the previous bits should be small compared to the standard deviation of height in $I$.  Intuitively, we are enforcing a low signal-to-noise ratio.

\begin{Def} \label{unpredictable}
Let $D$ be a distribution which produces bits in an online fashion and $s$ be the sequence of bits that have been produced immediately preceding an interval $I$.
Let $\E[A_s(I)]$ denote the expected payoff of an algorithm $A$ on interval $I$ 
(where the bits in $I$ are produced according to $D$ conditioned on having produced $s$ immediately before $I$).
Note that $A$ must fix its prediction for $I$ based solely on $s$ and before looking at any bits within $I$.

We say that $D$ is $\delta$-unpredictable if for all $A$, $s$ and $I$, $\E[A_s(I)] \leq \delta \cdot \sqrt{|I|}$.
\end{Def}

\begin{figure}
\includegraphics[width = 2.4in]{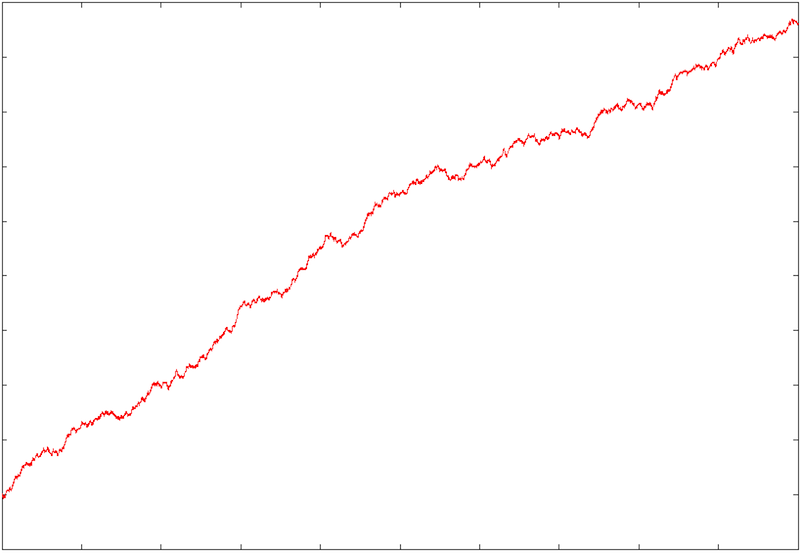} 
\includegraphics[width = 2.4in]{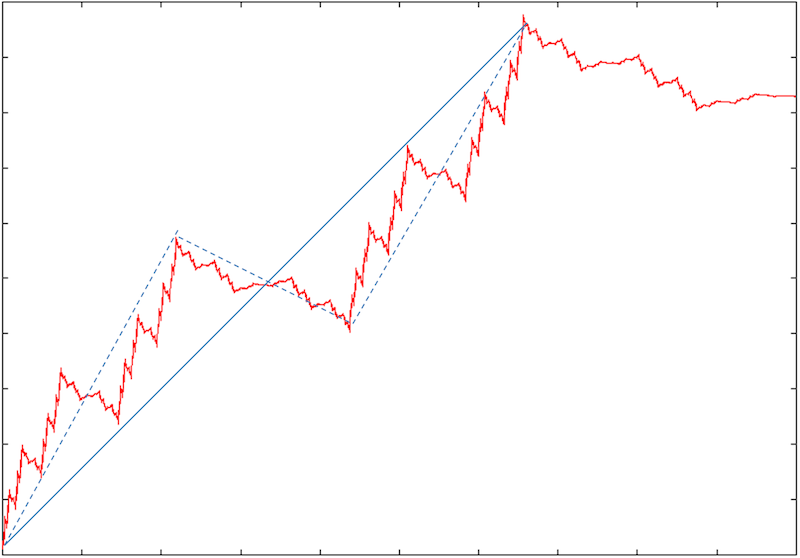}
\caption{Growth charts for two different types of adversarial sequences. The first is the cumulative plot of a random i.i.d sequence with a constant upward bias. The second is an $\alpha$-inverting sequence as in Definition \ref{inversion-deterministic}.  Note that the latter plots seems to change direction more significantly than the former.}
\end{figure}

For example an algorithm may notice a high density of $+1$'s and may decide to predict $+1$
for the next few bits (this would correspond to a ``buying'' a stock) for the next $x$ bits.  
Note that $\sqrt{x}$ is the standard deviation in the payoff of an algorithm for the uniform distribution on $x$ bit sequences and thus we are asking that 
the payoff of the algorithm for a $\delta$-unpredictable distribution is negligible compared to this standard deviation (we will in fact construct distributions where the
standard deviation is much higher than $\sqrt{x}$).   Roughly, this is equivalent to saying that the signal to noise ratio in any interval is negligible.

We ask what is the maximum deviation that can be achieved by a $\delta$-unpredictable distribution $D$.  We will look at
maximizing measures such as median deviation or mean deviation: $\E_{s \sim D} [|h(s)|]$ (we will show that our claims hold with respect to any of these measures). 

We show that there is a $\delta$-unpredictable distribution which achieves a deviation of $\sqrt{T} (1 + \Omega(\delta \log T))$.  Thus, the deviation 
can be $\omega(\sqrt{T})$ for $\delta = o(1)$.
  The distribution we construct is a variant of a 
discretization of FBM.  We also show that the highest deviation that can be achieved by a $\delta$-unpredictable distribution is $\sqrt{T} (1 + O(\delta \log T))$.
In addition, we construct a distribution which is a simple discretization of FBM and show that the deviation achieved 
by this distribution is $T^{1/2 + \Theta(\delta)}$.  Though this distribution is not
$\delta$-unpredictable, it satisfies a similar but weaker property.

A nice property of $\delta$-unpredictable distributions is that they are ``fractal-like'' in some sense.  We use the terms fractal-like and fractal somewhat interchangeably. 
Normally fractal is considered to be a self-similar recursive 
structure in Euclidean space (usually with non-integer dimension to 
exclude trivial patterns). Traditionally this has not been applied to 
bit sequences. Therefore we refrain from calling such sequences 
strictly a fractal.
To formalize our ``fractal-like" property, we first define a notion of inversion for a 
deterministic sequence.  The property essentially says that if in any interval there is a huge rise, then there must be a sub-interval where there must be a proportionally 
big fall and vice versa.

\begin{Def}[$\alpha$-Inversion] \label{inversion-deterministic}
Given a sequence $s \in \{-1, 1\}^T$, it is said to be $\alpha$-inverting if
for every interval $X$ within $[1, T]$ (of at least some constant length) there is a sub interval $Y$ such that  $h(s_X)$ and $h(s_Y)$ are of 
opposite sign and $|h(s_Y)|/|h(s_X)| \geq \alpha$.  Here by $s_I$ we mean the sequence $s$ restricted to interval $I$.

We refer to the largest feasible $\alpha$ as the inversion ratio of $s$.
\end{Def}

Observe that an $\alpha$-inverting sequence resembles a fractal in a certain sense.  To see this, note that in a sequence $s$ such that $h(s) > 0$,
if we locate the biggest contiguous rise, it may be divided into three parts $s_1 s_2 s_3$ where $s_2$ has a 
net downward slope and $s_1, s_3$ have a positive slope each. But one can recurse and divide each of the three substrings further into three parts each and thus the sequence has a recursive, self-similar
structure.

We show that any $\delta$-unpredictable distribution is $\alpha$-inverting in a certain sense.  Since we are dealing with a distribution rather than a deterministic sequence
we need an appropriate generalization of Definition \ref{inversion-deterministic} which is stated in Section \ref{results}.  It will be clear from the definition that
the highest entropy sequence we discussed earlier has a very small inversion ratio compared to $\delta$-unpredictable distributions.

\subsection{Main results} \label{results}

In this section we describe our main results in more detail.   As we mentioned earlier, the adversarial distributions we construct are closely related to and inspired from
FBM.

FBM with parameter $H$ is the unique continuous time, Gaussian process $B_H(t)$ which satisfies $B(0) = 0$, $\E[B_H(t)] = 0$ for all $t$ and has covariance function:

$$ \E[B_H(t) B_H(s)] = \frac{1}{2} (|t|^{2H} + |s|^{2H} - |t - s|^{2H} ) $$

The process $B_H$ is translation invariant and is
self-similar in the sense that 
$\{ B_H(at) : t \in \R \}$ is identical in distribution to $\{ a^H B_H(t) : t \in \R \}$ for all $a > 0$.  Furthermore, $B_H(t)$ is normally
distributed with variance $t^H$. Thus any interval of length $t$ has deviation about $t^H$.  The case $H = 0.5$ corresponds to the standard Brownian motion. 

The analysis of the FBM usually requires an understanding of integrated Wiener processes.  
The first adversarial distribution we construct is a discrete variant of the FBM 
that produces bits instead of real numbers.  We denote this distribution as {\sc Fractal Random Walk (FRW)}.

The sequence is constructed recursively in lengths that are powers of $2$. To produce a sequence of length $2n$, we concatenate two recursively constructed sequences of 
length $n$ each, and change the height of the
second sequence by a factor proportional to the height $h$ of the first sequence. This is done by flipping approximately $\delta h$ $(-1)$'s to 
$+1$'s if $h>0$ (and $+1$'s to $-1$ otherwise.)  A formal description of the construction appears in Section \ref{construct}.

While this lacks the translation invariance and the exact self-similarity properties of the FBM, it still has the property that any interval of size 
$t$ has deviation $t^{1/2 + \Theta(\delta)}$. 

To see this, note that if $h_1, h_2$ denote the heights of the two sequences that are concatenated to produce the sequence of length $2n$ after 
altering the second string then $E[h_1 h_2] = 2 \delta \E[h_1^2] = 2 \delta \E[H(n)^2]$ where $H(n)$ is a random
variable that denotes the height of a random sequence of length $n$ drawn from {\sc FRW}. So $\E[H(2n)^2] = E[(h_1 + h_2)^2] = E[h_1^2] + E[h_2^2] + 2E[h_1 h_2] = (2+\delta) E[H(n)^2]$. 
The recurrence works out to a root mean square deviation
$(\sqrt{E[H(n)^2]})$ of  about $n^{1/2 + \Theta(\delta)}$.  

This informal description skips over technical issues such as discretization.
Furthermore, extending this argument to show that the high deviation is achieved with constant probability is more complicated and is done in Theorem \ref{deviation-DFBM}.
Note that a constant probability bound for achieving a particular deviation is stronger than showing a high deviation in expectation (using Markov's inequality).
We note that this distribution is not $\delta$-unpredictable but satisfies a weaker property (Theorem \ref{unpredictable-DFBM}).   For completeness, we show that the FBM 
(continuous version) with $H = 1/2 + \delta$ is also not $\delta$-unpredictable in the strict sense (Claim \ref{FBM-not-delta}).
 We also note that
the highest entropy distribution is very poor in terms of $\delta$-unpredictability (Claim \ref{entropy-predictable}).

We construct another distribution, which we call {\sc Optimal Fractal Random Walk (Opt-FRW)} which has optimal trade-offs between deviation and predictability. 
The distribution {\sc Opt-FRW} is a simple but important twist on the above process where instead
of flipping $\delta \cdot h$ bits, we flip $\delta \cdot \sqrt{n}$ bits in the direction of $h$.

\begin{theorem} (Theorems \ref{unpredictable-OPT}, \ref{deviation-OPT} and \ref{deviation-upper})
The distribution {\sc Opt-FRW} is $O(\delta)$-unpredictable and achieves a deviation of  $\sqrt{T} (1 + \Omega(\delta \log T))$ with constant probability.
Further, no $\delta$-unpredictable distribution can achieve an expected deviation higher than $\sqrt{T} (1 + O(\delta \log T))$.
\end{theorem}

We now turn to formalizing the relationship between $\delta$-unpredictability and ``fractal-like'' property of a distribution.

For a deterministic sequence we show that an $\alpha$-inverting sequence with the highest deviation is a fractal.

\begin{theorem}\label{detinvertingisfractal} (Claim \ref{c})

Let $s$ be an $\alpha$-inverting sequence of length $t$ (Definition \ref{inversion-deterministic}), where $\alpha$ is bounded above by a constant.
  Then the highest deviation that can be achieved by $s$ for large $t$ is
$t^\theta$ where $\theta$ is the solution to the equation $1 =2((1+\alpha)/2)^{1/\theta} + \alpha^{1/\theta}$.  Furthermore, this deviation is actually achieved by an appropriately
designed fractal.

\end{theorem}

For distributions $D$ over sequences we  define the following variant of the earlier inversion rule.

\begin{Def}[$(\alpha, q)$-Inversion]  \label{inverting-dist}

A distribution $D$ is said to be $(\alpha, q)$-inverting if for any interval $X$ of at least some constant length) with median deviation $\Delta = \Omega(\delta \sqrt{|X|})$, 
with probability at least $q$  there is a sub interval $Y$ such that  $h(s_X)$ and $h(s_Y)$ are of opposite sign and $|h(s_Y)| \geq \alpha \cdot \Delta$. 
Here by $s_I$ we mean the sequence $s$ restricted to interval $I$. 
This should hold even if one conditions on a given history of bits seen before the interval $X$.

\end{Def}

We note (see Observation \ref{observationrandomwalk}) that a uniform random sequence is $(\alpha, q)$ inverting for some constants $\alpha,q$.  
 Further the probability parameter $q$ can be made as high as  $1 - \eps$  by reducing the inversion ratio $\alpha$ to $\Theta(1/\log(1/\eps))$. 

The following theorem establishes that every $\delta$-unpredictable distribution must be fractal-like in the sense that it is $(\Omega(1), \Omega(1))$ inverting.

\begin{theorem} \label{fractal-like} (Theorems \ref{a} and \ref{b})
For $\delta$ small enough, any $\delta$-unpredictable distribution is also $(\alpha,q)$-inverting for some constants $\alpha, q$.  
Further by dropping the inversion ratio $\alpha$ to $\Theta(1/\log T)$ the probability $q$ can be made as high as $1-1/T^{\Omega(1)}$ for all intervals of length at 
least $\Omega(\log T)$. Thus the condition holds with high probability simultaneously for all such intervals. 
\end{theorem}

\subsection{Related Work}

Many studies support the thesis that fractals occur naturally in several real world processes in diverse fields such as physics, finance and geography~\cite{mandelbrot1984fractal,davy1990some,mandelbrot2004fractals}. 
Ralph Elliot~\cite{frost1998elliott},
a professional accountant, suggested the use of fractal like ``waves'' in understanding financial markets.  Fractal models for finance have also been studied widely in the academic 
community. 
Fractional Brownian Motion (FBM) was introduced as a variant to the well known Brownian Motion by Mandelbrot and van Ness in \cite{mandelbrot1968fractional}. 
In addition to financial time series modeling, FBM has also found applications in the study 
of network traffic and fluid turbulence \cite{norros1995use,nualart2006fractional}.  

The reason for considering FBM rather than the standard Brownian motion for financial modelling was the observation that the distribution of financial time 
series is heavy-tailed \cite{bradley2003financial,rachev2005fat}.  This means that the deviations achieved are a bit higher than those expected for Brownian motion.
It has been argued that modeling S\&P500 price data according to FBM produces an estimated value of the 
Hurst parameter $H$ to be slightly over the $0.5$ value that corresponds to the standard Brownian Motion \cite{bayraktar2004estimating}.   
Values of $H > 0.5$ allow for long range (positive) correlations in the time series that results in a higher than normal deviation. 
Besides FBM other models such as $p$-stable distributions 
 and levy distributions \cite{voit2005,rachev2005fat,nolan2003stable} provide an alternate explanation for the heavy tailed nature of time series data by 
 allowing heavier tails for the price changes \emph{in each unit time} that 
 are \emph{independent} across time.  In contrast, the FBM uses normally distributed price changes in each unit time, and the high deviations are achieved by correlations across time.
  
Works such as \cite{rogers2002arbitrage,sottinen2003arbitrage}  have analyzed the level of arbitrage present in FBM.
The authors in \cite{FBMNorros} have analyzed the predictability
of the FBM using a different loss function from ours.  
Other researchers \cite{shafer2001probability,abernethy2012minimax} have studied the prediction problem as a game between an 
algorithm and an adversary, and derived that the optimal strategy for the adversary resembles a Brownian Motion.  The work in \cite{abernethy2012minimax} was inspired by 
 \cite{demarzo2006online} where the authors provide robust upper and lower bounds for
pricing European call options, under the no-arbitrage assumption when the price process is assumed to be \emph{discrete and discontinuous} as opposed to 
the Black Scholes model \cite{black1973pricing}
where the price process is taken to be continuous. 

\subsection{Discussion and Future work}

Note that our notion of $\delta$-unpredictable requires the algorithm to fix its prediction for an entire interval $I$ before looking at any of its bits.  A stronger notion of
unpredictability is to allow the algorithm to change its prediction for the interval after looking at bits within $I$.  In other words, at every point the algorithm tries to simply
predict the next bit, based on the bits it has seen so far.  One could ask what is most adversarial distribution in this setting which achieves a high deviation.  In this setting,
for any sequence $s$, a bounded regret algorithm such as Weighted Majority can achieve a payoff of $|h(s)| - c \sqrt{|s|}$ where $c := \sqrt{2/\pi}$ \cite{weighted-majority,cover-binary}. So for a distribution $D$ which achieves
typical deviation $k \sqrt{T}$, it is always possible to get a payoff of $(k - c) \sqrt{T}$.  It is also fairly straightforward to construct a distribution $D$ such that no algorithm
can achieve an expected payoff better than $(k - c) \sqrt{T}$ even when it predicts one bit at a time. We also note that while the distributions inspired by FBM have some
guarantees in terms of $\delta$-unpredictability, they perform poorly in this model when one is allowed to predict based on all previous bits (see Claim~\ref{d}).

One possible justification for our notion of $\delta$-unpredictability is that changing predictions very often may have a cost associated with it.  Although this may be a reasonable
assumption (at least for financial markets), it is only a conjecture at this point and we invite further comments on this issue.

An interesting direction for further research is to look for natural constraints on real world processes which provably result in the formation of fractal-like processes.

\section{Preliminaries}

Here is some common notation we use throughout the paper.  For a sequence of bits $s \in \{-1, 1\}^T$, $h(s)$ denotes the sum of bits in $s$ i.e. the \emph{height} of $s$.  
We refer to the magnitude of height as deviation.  

We will be working with several aggregate measures of deviation for a distribution such as median deviation (or generalized median),
mean deviation and root-mean-squared deviation ($\sqrt{\E_{s \sim D} [h(s)^2]}$). Note that mean deviation is no more than root-mean-squared deviation 
and the generalized median is bounded by mean deviation
up to constant factors using Markov's inequality (as long as the probability in generalized median is at least a constant).  We will prove our upper bounds for 
root-mean-squared deviation and lower bounds for generalized median and so they will hold for all measures up to constants.

We will typically denote random variables by capital letters and fixed sequences by small letters.
 
\section{Construction of Adversarial distributions} \label{construct}

In this section we formally construct our adversarial distributions.  
Each of these distributions has two parameters, $l$ which is the length of the sequence in the base case and $\delta > 0$.

We will construct the distributions inductively:
having constructed $D_\delta(n)$ we will show how to construct $D_\delta(2n)$ (the base case for $n=l$ is simply a random sequence in $\{-1, 1\}^l$).  In both cases below,
we describe the distribution 
$D_\delta(2n)$ in terms of how to generate a sequence $s \sim D_\delta(2n)$ given access to distribution $D_\delta(n)$.

{\sc Fractal Random Walk (\D{l})} $(2n)$

\begin{enumerate}

\item Generate sequences $s_1, s_2$ independently according to \D{l}$(n)$
\item If height of $s_1$ is positive, change exactly $\delta \cdot h(s_1)$ $-1$'s in $s_2$ to $1$ (if they exist, otherwise change as many as possible).  Similarly, 
if height of $s_1$ is negative, change exactly $\delta \cdot h(s_1)$ $1$'s in $s_2$ to $-1$ (if they exist).  Call the resulting sequence $s_2'$.
\item Set $s = s_1 \cdot s_2'$ i.e. the concatenation of $s_1$ and $s_2'$

\end{enumerate}

{\sc Optimum Fractal Random Walk (\O{l})}$(2n)$

\begin{enumerate}

\item Generate sequences $s_1, s_2$ independently according to \O{l}$(n)$
\item If height of $s_1$ is positive, change exactly $\delta \sqrt{n}$ $-1$'s in $s_2$ to $1$ (if they exist, otherwise change as many as possible).  Similarly, 
if height of $s_1$ is negative, change exactly $\delta \sqrt{n}$ $1$'s in $s_2$ to $-1$ (if they exist).  Call the resulting sequence $s_2'$.
\item Set $s = s_1 \cdot s_2'$ i.e. the concatenation of $s_1$ and $s_2'$

\end{enumerate}

{\bf Note: } Note that both distributions involve changing exactly $r$ bits in $s_2$ where $r$ is a real number.  Intuitively, we want to change each bit of the appropriate sign
in $s_2$ with probability $r/n$.  However, it is simpler to analyze the deviation of the distributions when we change exactly $r$ bits.  The fact that $r$ is a real number and not 
an integer will not make much difference since our base case $l$ will be an increasing function of $T$ (total number of bits to be produced) and so the discretization errors can
be safely ignored.

\subsection{High deviation}

In this section we show that the distributions we constructed achieve high deviation with constant probability.  What follows is a proof sketch for high deviation of distribution \D{i}.  Due to space constraints, 
the proof for \O{i} and for the intermediate claims appears in the appendix (Section \ref{omitted}).

\begin{theorem} \label{deviation-DFBM}
The distribution \D{l}$(T)$ achieves a deviation of $T^{1/2 + \Theta(\delta)}$ with probability at least $1/2 - \epsilon$ where $\epsilon \leq T^{-10}$.
\end{theorem}
\begin{Proof}
To analyze the height distribution of \D{l} it will be more convenient to define another process which is similar to \D{l} but which can assume integer values instead of bits.  

{\sc Augmented Fractal Random Walk (\C{l})} $(2n)$

\begin{enumerate}

\item Generate sequences $s_1, s_2$ independently according to \C{l}$(n)$
\item If height of $s_1$ is positive, change exactly $\delta \cdot h(s_1)$ $-1$'s in $s_2$ to $1$ (if they exist).  
Similarly, if height of $s_1$ is negative, change exactly $\delta \cdot h(s_2)$ $1$'s in $s_2$ to $-1$ (if they exist).  Call the resulting sequence $s_2'$.
\item {\bf Augment:} If there aren't enough $-1$'s to flip in $s_2$, then add $2$ to some of the numbers so that the increase in height is exactly $\delta \cdot h(s_1)$.  
Similarly for $1$'s.
\item Set $s = s_1 \cdot s_2'$ i.e. the concatenation of $s_1$ and $s_2'$

\end{enumerate}

For the random variable $S \sim$ \C{l}, we can exactly characterize the distribution of $h(S)$.

\begin{Claim} (Claim \ref{claim-1})
For $n = 2^i \cdot l$, $S \sim $\C{l}$(n)$, 
\begin{equation} \label{unroll}
h(S) = \displaystyle\sum_{U \subseteq [i]} r^{|U|} h(X_U)
\end{equation}
where $r = (1 + \delta)$ and each $X_U$ is independently and uniformly distributed in $\{-1, 1\}^l$.

\end{Claim}

We then apply the Berry-Esseen theorem (Theorem \ref{berry-esseen}) to show that the deviation of $|h($\C{l}$)|$ is high.  

\begin{lemma} (Lemma \ref{apply-BE})
Median of $|h($\C{l}$)|$ is $ n^{1 + \Omega(\delta)}$.
\end{lemma}

Next we show that the probability of executing step {\bf Augment} in \C{l} is exponentially small.  Note that when constructing a sequence of size $T$,
 the inductive steps of distribution \C{l} are executed at most $2 T$ times.  We show that when starting with sequences of size $l$ where $l = 100 \log T$,
 the probability that sequence $s_2$ doesn't have enough $1$'s or $-1$'s to flip at a particular stage is at most $T^{-10}$. 
  Thus, taking a union bound over all inductive steps, we  get the desired result.
 
 \begin{Claim} (Claim \ref{claim-2})
 The probability that step {\bf Augment} is executed at a particular step is at most $T^{-10}$.
 \end{Claim}
 
When the step {\bf Augment} is not executed, the distributions {\sc AFRW} and {\sc FRW} are identical.  Thus, the probability that the distribution \D{l}$(T)$ achieves a deviation of $T^{1/2 + \Theta(\delta)}$ is at least $ 1/2 - T^{-10}$.
%
%
%
%

\end{Proof}

\subsection{Unpredictability}

In this section we show that the distribution \O{l} is $\delta$-unpredictable.

We first observe that it suffices to work with \emph{aligned} intervals i.e. intervals which start and end at appropriate powers of $2$.

\begin{Def} (Aligned interval)

We assume here that $T$ is a power of $2$.  An aligned interval is one which is obtained by breaking $[1, T]$ into $2^i$ equal parts for $i \in [0, \log T]$ and picking one of the parts.
So for instance the first part is always $[1, 2^i]$.

In other words, an interval $[p+1, p+x]$ given by $p \in [0, T]$, $x \in [1, T- p]$ is said to be an aligned interval if $p = j \cdot 2^i$ and $x = 2^i$ for some
$i \in [0, \log T]$ and $j \in [0, T - 2^i]$.

\end{Def}

\begin{Claim} \label{align}
If distribution $D(T)$ is $\epsilon$-unpredictable with respect to all aligned intervals then it is $c \cdot \epsilon$-unpredictable with respect to all intervals,
where $c := \frac{\sqrt{2}}{\sqrt{2} - 1}$.
\end{Claim}

The proof of Claim \ref{align} is fairly straightforward and is moved to the appendix (Claim \ref{align-proof}).

\begin{theorem} \label{unpredictable-OPT}
The distribution \O{l} is $O(\delta)$-unpredictable.
\end{theorem}
\begin{Proof} [Sketch]

It can be shown that the process \O{l} has very similar properties if in Step $2$ of the construction, instead of changing
exactly $\delta \cdot \sqrt{n}$ bits in $s_2$ we change each bit (of appropriate sign) in $s_2$ with probability $\frac{\delta}{\sqrt{n}}$.
Here we assume this fact without proving it.

We need to show that for every $A$, $s$ and $I$, $\E[A_s(I)] \leq O(\delta) \cdot \sqrt{|I|}$ where $s$ and $I$ are as in Definition \ref{unpredictable}.
We may assume that $I$ is an aligned interval (Claim \ref{align}).

  From the construction it is clear that $\E[A_s(I)]$ is largest when $h(s) = |s|$ or $h(s) = -|s|$ i.e. all
the bits before $I$ are of the same sign.  Without loss of generality assume $h_s = s$. Also, if there were no prefix (i.e. $|s| = 0$) then 
$\E[A_s(I)] = 0$ since the construction is symmetric.  To provide an upper bound on $\E[A_s(I)]$ we simply need to bound the expected number of $-1$'s which are changed
to $+1$'s due to the existence of $s$.  We will use a simple union bound on the total probability of changing a $-1$ to a $1$ according to the construction.  This probability can
be split into $2$ parts, the first which occurs because of bit sequences immediately preceding $I$ of length less than $I$ and the second because of bit sequences 
immediately preceding $I$ of length more than $I$.  For sequences of the first kind, the number of bits changed in $I$ is exactly $\delta \cdot \sqrt{l}$ while for sequences of the 
second kind we may assume that the expected number of bits changed in $I$ is $\frac{\delta \cdot |I|}{\sqrt{l}}$ where $l$ is the length of the bit sequence under discussion.  
Thus, the total probability is bounded by:-

$$ \displaystyle \sum_{i = 1}^\infty (\min(|I|, 2^i) \cdot \delta) / \sqrt{2^i} = \sum_{i=1}^{\log |I|} \delta \cdot \sqrt{2^i} \ + \sum_{i = \log |I| + 1}^\infty \delta \cdot \frac{|I|}{\sqrt{2^i}} $$

Both terms can be bounded by $ \delta \cdot \sqrt{|I|} \cdot \sum_{i = 0}^\infty 1/\sqrt{2^i}$ and so the combined sum is at most $O(\delta) \cdot \sqrt{|I|}$.

\end{Proof}

\section{Deviation upper bound for Adversarial Distributions}

In this section we prove that the deviation achieved by {\sc Opt-FRW} is essentially the best possible for a $\delta$-unpredictable distribution up to a constant factor.

\begin{theorem} \label{deviation-upper}
The highest Root-Mean-Square deviation that can be achieved by a $\delta$-unpredictable distribution on sequences of length $T$ is $\sqrt{T} (1 + O(\delta)) \log T$.
\end{theorem}
\begin{Proof}

Let $\mathcal{D}_\delta(T)$ be the set of all $\delta$-unpredictable distributions over sequences of length $T$, and let 
$h_n = \max_{D \in \mathcal{D}_\delta(n)}  \E_{s \sim D} [h(s)^2]$.  Clearly, $h_1 = 1$.  We need to show that $\sqrt{h_T} = \sqrt{T} (1 + O(\delta)) \log T$.

Let $D(T)$ be a $\delta$-unpredictable distribution which maximizes $\E_{s \sim D} [h(s)^2]$.
Given a sequence $s \sim D$, we write $s = s_1 s_2$ where $s_1$ and $s_2$ are of length $n/2$ each.  Then we have,

\begin{align*}
h_n = \E_{s \sim D} [h(s)^2] & = \E[(h(s_1) + h(s_2))^2] \\
& = \E[h(s_1)^2] + \E[h(s_2)^2] + 2 \E[h(s_1) h(s_2)] \\
& \leq 2 h_{n/2} \ + 2 \sum_{x = 0}^{n/2} \Pr[h(s_1) = x] \cdot x \cdot \E[h(s_2) \ | \ h(s_1) = x] \\
& \leq 2 h_{n/2} + 2 \delta \sqrt{n/2} \sum_{x = 0}^{n/2} \Pr[h(s_1) = x] \cdot |x| \\
& = 2 h_{n/2} + \delta \cdot \sqrt{2n} \cdot \E[|h(s_1)|] \\
& \leq 2 h_{n/2} + \delta \cdot \sqrt{2n} \cdot \sqrt{\E[h(s_1)^2]} \\
& \leq 2 h_{n/2} + \delta \cdot \sqrt{2n} \cdot \sqrt{h_{n/2}} \\
\end{align*}

The first inequality follows from the definition of $h_{n/2}$.  The second inequality follows from the fact that the distribution of $s_2$ is also $\delta$-unpredictable.

Let's substitute, $g_n^2 := h_n/n$.  Then $h_{n/2} = (n g_{n/2}^2)/2$ and $\sqrt{h_{n/2}} = \sqrt{n/2} \sqrt{g_n^2/2}$.  Thus, we get

\begin{align*}
n g_n^2 & \leq &\  n g_{n/2}^2 + \delta n \sqrt{g_{n/2}^2} \\
\implies  \ g_n^2 & \leq &\ g_{n/2}^2 + \delta \sqrt{g_{n/2}^2} \ \ \leq \ \ \ \(\sqrt{g_{n/2}^2} + \delta/2\)^2 \\
\implies \ g_n & \leq & \ g_{n/2} + \delta/2 
\end{align*}

Since $g_1 = 1$, this gives the upper bound $g_n \leq 1 + (\delta/2) \log n $.  This implies $\sqrt{h_n} = \sqrt{\E_{A \sim D} [h_A^2]} \leq \sqrt{n} (1 + \delta/2) \log n$.

\end{Proof}

\section{Acknowledgements}
We thank Alex Andoni and Samuel Ieong for useful discussions.

\bibliographystyle{splncs}
\bibliography{fractal}

\appendix

\section{Fractal nature of Adversarial Distributions}

Here we show that any distribution which is $\delta$-unpredictable must have a fractal like nature (Theorem \ref{fractal-like}).  We will first show that $\delta$-unpredictable distributions are also unpredictable
in a slightly stronger sense.

\begin{Def} [Adaptive interval algorithm] 
An interval prediction algorithm is said to be adaptive if it can choose to stop making predictions on interval $I$ at any point within $I$ based on the bits it has seen so far.
Note that we do not allow the prediction of the algorithm to depend on the bits in $I$, the only decision the algorithm can make based on bits in $I$ is to stop predicting earlier
than the end point of $I$.
\end{Def}

\begin{Def} [Adaptively $\delta$-unpredictable] \label{adapt}
A distribution $D$ is said to be adaptively $\delta$-unpredictable if for any adaptive algorithm $A$, sequence of bits $s$ and interval $I$, $\E[A_s(I)] \leq \delta \cdot \sqrt{l}$
where $l$ is the expected time for which $A$ continues making a prediction in $I$.

Here the bits in $I$ are produced according to $D$ conditioned on having produced $s$ immediately before $I$, similarly as in Definition \ref{unpredictable}.
\end{Def}

\begin{theorem} \label{a}
A $\delta$-predictable distribution is also adaptively $O(\delta)$-unpredictable.
\end{theorem}
\begin{Proof}

Let $D$ be a $\delta$-predictable distribution and $A'$ an adaptive interval algorithm.  We first show that $\E[A'_s(I)] \leq 2\delta \cdot \sqrt{|I|}$ i.e. we replace the \emph{expected}
time for which $A'$ continues making a prediction in $I$ by the \emph{maximum} time for which it makes a prediction.

We will construct a non-adaptive algorithm $A$ such that $\E[|A_s(I) - A'_s(I)|] \leq \delta \cdot \sqrt{|I|}$.  Since $\E[A_s(I)] \leq \delta \cdot \sqrt{|I|}$ 
($D$ is $\delta$-unpredictable) this implies that $\E[A'_s(I)] \leq 2\delta \cdot \sqrt{|I|}$

Let $p_u$ be the probability of producing a sequence of bits $u$ as a prefix in $I$ according to distribution $D$.  Let $E$ be the set of sequences $u$ such that the algorithm
$A'$ stops making predictions on seeing $u$.  Then $\sum_{u \in E} p_u = 1$.

Let $P_u(A)$ denote the expected payoff of $A$ on the remaining part of $I$ conditioned on the event that $A'$ has stopped making predictions.  Then 
$P_u(A) \leq \delta \cdot \sqrt{|I| - |u|} \leq \delta \cdot \sqrt{|I|}$.  Thus, $\E[|A_s(I) - A'_s(I)|] \leq \sum_{u \in E} p_u \cdot \delta \cdot \sqrt{|I|} = \delta \cdot \sqrt{|I|}$.

Now we extend the proof to the case where $A'$ makes a prediction for expected time $x$ rather than maximum time $x$.  

Let $q_i$ be the probability that $A'$ makes a prediction for time more thant $2^i x$.  By Markov's inequality, $q_i \leq 2^{-i}$.  Also, $q_i = \sum_{u \in E: |u| = 2^i} p_u$, where
$p_u$ is as defined above.  We will bound the payoff of $A'$ in phases where the $i^{th}$ phase consists of bits between $2^i x$ to $2^{i+1}x$ from the start of $I$, and show that it 
is at most $2 \delta \cdot q_i \cdot \sqrt {2^i x}$.  For a fixed sequence $u$, the payoff of algorithm $A'$  in phase $i$ conditioned on having seen $u$
 is at most $2 \delta \sqrt{2^i x}$ (proved above).  Thus, the total payoff of $A'$ in phase $i$ is at most $2 \delta \cdot q_i \cdot \sqrt {2^i x}$. 
 Finally, the expected payoff of $A'$ over all phases is at most:

$ \sum_i 2 \delta q_i \sqrt {2^i x} \leq  2 \delta \cdot \sum_i \sqrt {2^i x}/2^i \leq O(\delta) \cdot \sqrt{x})$

which proves that $D$ is adaptively $O(\delta)$-unpredictable.

\end{Proof}

Now we turn to showing that any adaptively $\delta$-unpredictable distribution has a fractal like nature.

\begin{theorem} \label{b}
If a distribution over $T$ bit sequences is adaptively $\delta$-unpredictable (Definition \ref{adapt}) then it is $(\alpha, q)$-inverting for some constants $\alpha, q$. 
Further by dropping the inversion ratio $\alpha$ to $\Theta(1/\log T)$ the probability $q$ can be made as high as $1-1/T^{\Omega(1)}$ for all intervals of length at 
least $\Omega(\log T)$. Thus the condition holds with high probability simultaneously for all such intervals. 
\end{theorem}
\begin{Proof}

For a certain given history of bits consider the interval $I$. Let $h(I)$ denote the random variable that denotes the height of this interval. 
Let $\theta, p$ be such that the deviation in $I$ exceeds $\theta$ with constant probability $p$ (this generalizes the case when $\theta$ is the median deviation.) 

We will show that some prefixes of $I$ must achieve height at least $\alpha \theta$ and $-\alpha \theta$ each with constant probability (where $\alpha < 1/2$ is a constant).
 To show this, note that
either $h \ge \theta$ or $h \le -\theta$ with probability at least $p/2$. Assume it is the former without loss of generality.  
So we only need to prove that  $h \le -\alpha \theta$ with  probability at least $p/4$.  
Assume the contrary and we will see that the interval cannot be $\delta$-unpredictable.  

Consider a prediction algorithm that predicts $+1$ for the interval but adaptively terminates its betting if the height drops to $-\alpha \theta$ 
 or if the height exceeds $2 \alpha \theta$, whichever happens first.  
Since the algorithm hits the lower limit of  $-\alpha \theta$ only with probability at most $p/4$, so with at least probability $p/4$ it must realize the 
upper limit (payoff) of $2 \alpha \theta$ (since $2 \alpha < 1$). In all remaining cases the payoff is at least $-\alpha \theta$.  So the expected payoff is at 
least $(p/4)(2 \alpha \theta) - (p/4)(\alpha \theta)$ which needs to be at most $\delta \sqrt x$.  
This is not possible if $\alpha \le 1/2$ and $\theta = \Omega(\frac{\delta \cdot \sqrt{|I|}}{p})$. 
Thus if the height in an interval has high magnitude with constant probability, it must reach in either direction with constant probability.

To convert this into a high probability argument, we will use (at most) $s$ iterations of the above prediction algorithm each with limits that depend on $\theta/s$ 
instead of $\theta$. Each iteration has limits of $2 \alpha \theta/s$ and $- \alpha \theta/s$ on the sum of bits seen during its execution. 
The next iteration is initiated only if either of the upper or lower limit is reached  in the previous iteration and if not all $|I|$ bits in the full interval are exhausted. 
From the previous argument, conditioned on the event that a certain iteration is initiated,  if an iteration is executed for expected time $O(|I|/s)$ and 
hits the upper limit with probability $p/2$ then it must also hit the lower limit with probability $p/4$. Since the final height exceeds $\alpha \theta$ with constant 
probability $p$, in such cases all $s$ iterations have been initiated.  Since there are at most $s$ iterations and all are initiated with constant probability, at least 
half of them must have an expected length of $O(|I|/s)$ conditioned on the event that they are initiated; otherwise the total expected time of all the $s$ iterations will exceed $x$.

Conditioned on the event that the $i^{th}$ iteration is initiated, with probability $p$ it must hit at least one of its two limits; otherwise the total height will 
not reach $2 \alpha \theta$ with probability $p$. So conditioned on the event that the $i^{th}$ iteration is initiated, for at least half the iterations, it must hit 
the lower limit (and upper limit) with probability at least $p/4$. So conditioned on the event that all $s$ iterations are initiated the probability that none of 
them hit the lower limit and also the upper limit is at most $(p/4)^{s/2}$.

Thus, it follows that by choosing $s = \Theta(1)$, we get an $\alpha$ inversion for constant $\alpha$ with constant probability.  This proves the first part of the theorem.

For the second part, note that with probability at least $1 - (p/4)^{s/2}$ either the final height is less than $2 \alpha \theta$ or some subinterval has height $-\alpha \theta/s$. For $s = \Theta(\log T)$ 
the probability that the final height exceeds $\theta$ and there is no inversion of height $\le -\alpha/s \theta$ is negligible. 
\end{Proof}

\section{Omitted Proofs}\label{omitted}

\begin{theorem} \label{deviation-OPT}
The distribution \O{l}$(T)$ achieves a deviation of $\sqrt{T} (1 + \Omega(\delta \log T))$ with constant probability for  $l := T^{-3/4}$.
\end{theorem}
\begin{Proof}

To prove the theorem it will be more convenient to define another process which is similar to \O{l} but which can assume integer values instead of bits.  

{\sc Augmented Optimum Fractal Random Walk (\A{l})}$(2n)$
\begin{enumerate}

\item Generate sequences $s_1, s_2 \in \{-1, 1\}^n$ independently according to \A{l}$(n)$
\item If height of $s_1$ is positive, change exactly $\delta \cdot \sqrt{n}$ $-1$'s in $s_2$ to $1$ (if they exist).  
Similarly, if height of $s_1$ is negative, change exactly $\delta \cdot \sqrt{n}$ $1$'s in $s_2$ to $-1$ (if they exist).  Call the resulting sequence $s_2'$.
\item {\bf Augment:} If there aren't enough $(-1)$'s to flip in $s_2$, then add $2$ to some of the numbers so that the increase in height is exactly $\delta \cdot \sqrt{n}$.
  Similarly for $1$'s.
\item Set $s = s_1 \cdot s_2'$ i.e. the concatenation of $s_1$ and $s_2'$

\end{enumerate}

First we observe that when $l = T^{-3/4}$, the probability of executing step {\bf Augment} is exponentially small in $T$.  To see this note that if all the base sequences of length
$l$ have at least $c(T) := \delta \sqrt{T} \log T$ $(-1)$'s and at least $c(T)$ $1$'s then the step {\bf Augment} is never called.  This is because every inductive step removes
at most $\delta {\sqrt{T}}$ $1$'s or $-1$'s at each stage and the number of times a base sequence is modified is at most $\log (T/l) \leq \log T$.  Now note that by Chernoff bound, 
probability that a given base sequence does not have $c(T)$ $1$'s or $(-1)$'s is exponentially small in $T$.  Finally note that the number of base sequences is at most $T/l$,
so we can simply take a union bound over all of them.

For brevity, let $D := $\O{l} and $D' := $\A{l}.  The next observation is that it suffices to prove that $\E[|h(D')|]$ 
is $\sqrt{T} (1 + \Omega(\delta \log T))$ and $\E[h(D')^2] = O(\E[|h(D')|]^2)$ to prove the theorem.  To see this, let $NA$ be the event that the step 
{\bf Augment} is never executed at any point in the construction, then we have:-

\begin{align*}
\E_{s \sim D}[|h(s)|] & \geq \E_{s \sim D}[|h(s)| \ | \ NA] \\  
& = \E_{s \sim D'} [|h(s)|] \ | \ NA] \\
& = \E_{s \sim D'} [|h(s)|] - \Pr[NA] \cdot \max_{s \in D'} |h(s)|
\end{align*}

We already saw that $\Pr[NA]$ is exponentially small in $T$.  Note that the maximum value of $|h(s)|$ is at most 
$T + \delta (\sqrt{T/2} + 2 \sqrt{T/4}) + 4 \sqrt{T/8} + \ldots + 2^{T/l - 1} \sqrt{l}$
which is bounded by a polynomial in $T$.  Thus, if $\E[|h(D')|]$ is $\sqrt{T} (1 + \Omega(\delta \log T))$ then so is $\E[|h(D)|]$.
It is also easy to see that the maximum value of $h(s)^2$ is polynomial in $T$.  This fact combined with our assumption about $D'$, $\E[h(D')^2] = O(\E[|h(D')|]^2)$
 implies that $\E[h(D)^2] = O(\E[|h(D)|]^2) $.  Applying Lemma \ref{e} to distribution $D$ we get that deviation $\sqrt{T} (1 + \Omega(\delta \log T))$ is achieved with constant
 probability as required.
 
 So to reiterate, we need to prove two things:-
 
 \begin{itemize}
   \item $\E[|h(D')|]$ is $\sqrt{T} (1 + \Omega(\delta \log T))$
   \item  $\E[h(D')^2] = O(\E[|h(D')|]^2)$
 \end{itemize}

From now on, we denote by $S_T$ a random sequence $S$ drawn from the distribution $D'(T)$.  The random variable $h(S_T)$ can be written as 
$h(A_{T/2}) + h(B_{T/2}) + R$ where $R$ is $\delta \sqrt{T/2}$ if $h(A_{T/2}) > 0$ and $- \delta \sqrt{T/2}$ otherwise.  Here the pairs of variables
$(A_{T/2}, B_{T/2})$ and $(B_{T/2}, R)$ are independent.  Now define $h_T := h(S_T)$.  We see that,

\begin{align*}
\E[h_T^2] & = \E[h(S_T)^2] \\
& = \E[(h(A_{T/2}) + h(B_{T/2}) + R)^2] \\
& = \E[h(A_{T/2})^2] + \E[h(B_{T/2})^2] + \E[R^2] + 2 \E[h_A R] \\
& = 2 \E[h_{T/2}^2] + \delta^2 \sqrt{T/2} + 2 \delta \sqrt{T/2} \E[|h_A|] \\
& \geq 2 \E[h_{T/2}^2] + \delta \sqrt{2T} \E[|h_{T/2}|]
\end{align*} 

The following claim gives a lower bound for $\E[|h_{T}|]$.

\begin{Claim} \label{cauchy}
$$ \E[|h_{T}|] \geq \frac{\E[h_{T}^2]^2}{\E[h_{T}^4]^{3/4}}  $$
\end{Claim}
\begin{Proof}
Let the random variables $X, Y$ be defined as $X := |h_T|^{1/2}, Y := |h_T|^{3/2}$.  By Cauchy-Schwartz,

\begin{align*}
\E[XY]^2  & \leq & \E[X^2] \cdot \E[Y^2] \\
\implies \ \E[h_T^2]^2  & \leq & \E[|h_T|] \cdot \E[|h_T|^3] \\
& \leq & \E[|h_T|] \cdot \E[h_T^4]^{3/4} \\
\implies \ \E[|h_T|] & \geq & \frac{\E[h_{T}^2]^2}{\E[h_{T}^4]^{3/4}} 
\end{align*} 
\end{Proof}

Thus, we can say that 

\begin{align*}
\E[h_T^2] & \geq 2 \E[h_{T/2}^2] + \delta \sqrt{2T} \E[|h_{T/2}|] \\
& \geq 2 \E[h_{T/2}^2] + \delta \sqrt{2T} \frac{\E[h_{T/2}^2]^2}{\E[h_{T/2}^4]^{3/4}} \\
& =  2 \E[h_{T/2}^2] + \delta \sqrt{2T} \sqrt{\E[h_{T/2}^2]} \left(\frac{\E[h_{T/2}^2]}{\E[h_{T/2}^4]}\right)^{3/4} 
\end{align*}

First, let's complete the proof assuming that $\frac{\E[h_{T}^2]^2}{\E[h_{T}^4]} \geq C$ for all $T$ where $C$ is an absolute constant.  Let's substitute $g_T^2 := \E[h_T^2]/T$. Then,

\begin{align*}
\E[h_T^2] & \geq & 2 \E[h_{T/2}^2] + \Omega(\delta) \sqrt{2T} \sqrt{\E[h_{T/2}^2]} \\
\implies \ T \cdot g_T^2 & \geq & 2 \cdot (T/2) \cdot g_{T/2}^2  + \Omega(\delta) \cdot \sqrt{2T} \cdot \sqrt{T/2} \cdot \sqrt{g_{T/2}^2} \\
\implies \ g_T^2 & \geq & g_{T/2}^2 + \Omega(\delta) \cdot g_{T/2} \\
& = & (g_{T/2} + \Omega(\delta))^2 - O(\delta^2) \\
\implies \ g_T & \geq & g_{T/2} + \Omega(\delta)
\end{align*}

For the base case, we have $\E[h_l^2] = l$, thus $g_l^2 = 1$.  Thus, 

$$ g_T \geq 1 + \Omega(\delta) \cdot \log (T/l) = 1 + \Omega(\delta) \cdot \log T^{1/4}  = 1 + \Omega(\delta \log T)$$

Thus, $\E[h_T^2] \geq T \cdot g_T^2 = T (1 + \Omega(\delta \log T)^2)$.  By Lemma \ref{cauchy} and Lemma \ref{ratio}, 
this implies $\E[|h_T|] \geq \sqrt{T} (1 + \Omega(\delta \log T))$.  These statements together imply both the guarantees we set out to prove about $D'$.

It remains to prove the following lemma.

\begin{lemma} \label{ratio}
$\frac{\E[h_{T}^2]^2}{\E[h_{T}^4]} \geq C$ for all $T$ where $C$ is an absolute constant.
\end{lemma}
\begin{Proof}
Recall that for $S_T$ drawn according to \A{l}$(T)$, we have $h(S_T) = h(A_{T/2}) + h(B_{T/2}) + R$.  We already saw that $\E[h_T^2] \geq 2 \E[h_{T/2}^2]$.
Let $r_T := \frac{\E[h_{T}^4]}{\E[h_{T}^2]^2}$.  We need to show that $r_T \leq C$.  We have,

$$ r_T = \frac{\E[h_{T}^4]}{\E[h_{T}^2]^2} \geq \frac{\E[h_{T}^4]}{4 \E[h_{T/2}^2]^2} $$

Now, let's write a recurrence for $\E[h_{T}^4]$.  

\begin{align*}
\E[h_T^4] & = \E[h(S_T^4)] \\
& = \E[(h_A + h(s_1) + R)^4] \\
& = \E[h_A^4] + \E[h(s_1)^4] + \E[R^4] + 6 \E[h_A^2 h(s_1)^2] + 6 \E[h_A^2 R^2] + \\ & 6 \E[h(s_1)^2 R^2] + 4 \E[h_A R^3] + 4 \E[h_A^3 R] \\
& = 2 \E[h_{T/2}^4] + \delta^4 (T/2)^2 + 6 \E[h_{T/2}^2]^2 + 12 \delta^2 (T/2) \E[h_{T/2}^2] + \\ & 4 \delta^3 (T/2)^{3/2} \E[|h_{T/2}|] + 4 \delta \sqrt{T/2} \E[|h_{T/2}|^3] \\
& \leq 2 \E[h_{T/2}^4] + O(\delta^4 T^2) + 6 \E[h_{T/2}^2]^2 + O(\delta^2 T \E[h_{T/2}^2]) + \\ & O(\delta^3 T^{3/2}) \sqrt{\E[h_{T/2}^2]} + O(\delta \sqrt{T}) \E[h_{T/2}^4]^{3/4} \\
\end{align*}

Dividing both sides by $4 \E[h_{T/2}^2]^2$ and using the fact that $\E[h_T^4] \geq \E[h_T^2]^2 \geq T^2$, we get:-

\begin{align*}
r_T & \leq \frac{\E[h_{T}^4]}{4 \E[h_{T/2}^2]^2} \\
& \leq  (1/2) \cdot r_{T/2} + O(\delta^4) + (3/2) + O(\delta^2) + O(\delta^3) + O(\delta) r_{T/2}^{3/4} \\
& \leq (3/4) \cdot r_{T/2} + O(1)
\end{align*}

which is clearly bounded above by an absolute constant for all $T$.

\end{Proof}

Thus, the theorem is proved.

\end{Proof}

\begin{theorem} \label{unpredictable-DFBM}
The distribution $D := $ \D{l} is $O(\delta)$-unpredictable in a weak sense i.e. $\E_{s, I}[A_s(I)] \leq O(\delta) \cdot h_{|I|}$ where 
$h_n := \E_{s \sim D(n)}[|h(s)|]$.  Here $s$, $I$ and $A$ are as in Definition \ref{unpredictable}.  Note that the expectation on the left is taken over $I$ as well as the prefix
$s$ as opposed to Definition \ref{unpredictable} where $s$ is fixed and the expectation is over $I$ only.
\end{theorem}
\begin{Proof} [Sketch]

It can be shown that the process \D{l} has very similar properties if in Step $2$ of the construction, instead of changing
exactly $\delta \cdot h(s_1)$ bits in $s_2$ we change each bit (of appropriate sign) in $s_2$ with probability $\frac{\delta \cdot h(s_1)}{n}$.  
Here we assume this fact without proving it.

We need to show that $\E_{s, I}[A_s(I)] \leq O(\delta) \cdot h_{|I|}$.  We may assume that $I$ is an aligned interval (Claim \ref{align}).
Let $s(i)$ be the suffix of length $i$ in $s$.  Then,

\begin{align*}
& \E_{s}[ \E_{I} [ A_{s}(I) ] \\ 
 \leq & \sum_{i = 0}^{\log |I|} \delta \cdot \E_{s} [|h(s(2^i))|] + \sum_{i = \log |I| + 1}^{\infty} \frac{\delta \cdot \E_{s} [|h(s(2^i)|] \cdot |I|}{2^i} \\
\leq & O(\delta) \cdot \E_{s} [|h(s(|I|))|] +  \\ & O(\delta) \cdot \E_s [|h(s(|I|))|] (1/2^{1/4} + 1/4^{1/4} + \ldots) \\
\leq & O(\delta) \cdot \E_{s} [|h(s(|I|))|] \\
= & O(\delta) \cdot h_{|I|}
\end{align*}

where the second inequality uses $h_T = T^{1/2 + \Theta(\delta)} \leq T^{3/4}$ (Theorem  \ref{deviation-DFBM}).

\end{Proof}

\begin{Claim} \label{claim-1}
For $n = 2^i \cdot l$, $S \sim $\C{l}$(n)$, 
\begin{equation} \label{unroll}
h(S) = \displaystyle\sum_{U \subseteq [i]} r^{|U|} h(X_U)
\end{equation}
where $r = (1 + \delta)$ and each $X_U$ is independently and uniformly distributed in $\{-1, 1\}^l$.

\end{Claim}
\begin{Proof}
We will prove the claim by induction on $i$.  For $i=0$, the claim clearly holds.

Assume that the claim holds for $i=k$, and let $n := 2^{i + 1} \cdot l$. Let $S = S_1 \cdot S_2'$ be the sequence produced by the distribution as described above where $S_1$ and $S_2'$ 
are random sequences of length $n/2$ each. 
 Because of step {\bf Augment}, it is clear that
$h(S_2') = h(S_2) + \delta h(S_1)$ which means $h(S) = (1 + \delta) h(S_1) + h(S_2)$.  Thus,

\begin{align*}
h(S) & = r h(S_1) + h(S_2) \\
& = \displaystyle \sum_{U \subseteq [k]} r (r^{|U|} h(X_U)) +  \displaystyle\sum_{V \subseteq [k]} r (r^{|V|} h(X_V)) \\
& = \displaystyle \sum_{U \subseteq [k]} r^{|U \cup \{k+1\}|} h(X_{U \cup \{k+1\}}) +  \displaystyle\sum_{V \subseteq [k]} r (r^{|V|} h(X_V)) \\
& = \displaystyle \sum_{U \subseteq [k + 1]}  (r^{|U|} h(X_U)) 
\end{align*}
\end{Proof}

\begin{lemma} \label{apply-BE}
Median of $|h($\C{l}$)|$ is $ n^{1 + \Omega(\delta)}$.
\end{lemma}
\begin{Proof}
In the notation of Theorem \ref{berry-esseen} we think of each term in Equation \ref{unroll} as a random variable.  There are exactly $2^i$ terms.  It is clear that 
$\E[X_S] = 0$ for all $S \subseteq [i]$.  Also, $\E[X_S^2] = r^{2|S|} \cdot l$ and $\E[|X_S^3|] = r^{3 |S|} \cdot l$, 

$$ \sigma^2 := \sum_S r^{2 |S|} \cdot l = (r^2 + 1)^i \cdot l \approx (2 + 2 \delta)^i \cdot l \approx n^{1 + \Theta(\delta)} $$ 

Also, $\max_S \rho_S/\sigma_S = \max_S r^{|S|} = r^i \approx n^{\Theta(\delta)}$.
Thus, we have

\begin{equation*}
|S - N(0, \sigma^2)| \leq n^{-1/2} \cdot n^{O(\delta)}  \leq n^{- \Omega(1)}
\end{equation*}

Thus, the distribution of $|h($\C{l}$)|$ is very close to a half-normal distribution with variance $\sigma^2$ and thus the median of $|h($\C{l}$)|$ is
$\Omega(\sigma) = n^{1/2 + \Omega(\delta)}$.
\end{Proof}

 \begin{Claim} \label{claim-2}
 The probability that step {\bf Augment} is executed at a particular step is at most $T^{-10}$.
 \end{Claim}
 \begin{Proof}
 This is a simple application of Theorem \ref{hoeffding}.  
 
 Let's say we are at the step where the length of the sequences is $n := 2^i \cdot l$.
 We consider random variables $Y_S := r^|S| \cdot X_S$ where $X_S$ is as in Lemma \ref{apply-BE}. Observe that a single random variable $Y_S$ is actually
 a sum of $l$ independent random variables each of which take values in $\{-r^{|S|}, r^{|S|}\}$.  Let us denote these random variables as $Y_{S, i}$ so that
 $ X_S = \sum_{i = 1}^l Y_{S, i}$.  Thus, in the notation of Theorem \ref{hoeffding},
  
 $$ \sum_{S, i} (b_{S, i} - a_{S, i})^2 = 4 \sum_{S, i} r^{2 |S|} = 4 l \sum_{S} r^{2 |S|} = n^{1 + O(\delta)} $$
 
 where $\sigma^2$ is as in Lemma \ref{apply-BE}.
 
 The step {\bf Augment} is executed only when $h_S \geq n - \delta n$ or $h_S \leq -(n - \delta n)$.  Thus, we have the bound
 
  \begin{equation*}
 \Pr[|h_S| \geq n - \delta n] \ \leq \Pr[|h_S| \geq n/2]  \
  \leq 2 \cdot \exp\(- \frac{n^2}{n^{1 + O(\delta)}}\) \
  \leq 2 \cdot \exp\(- \frac{l^2}{l^{1 + O(\delta)}}\) \
  \leq T^{-10}
 \end{equation*}
 
 as desired.
 
 \end{Proof}

\begin{Claim} \label{align-proof}
If distribution $D(T)$ is $\epsilon$-unpredictable with respect to all aligned intervals then it is $c \cdot \epsilon$-unpredictable with respect to all intervals,
where $c := \frac{\sqrt{2}}{\sqrt{2} - 1}$.
\end{Claim}
\begin{Proof}
Consider an interval $I$ of size $x$.  If $I$ is an aligned interval we are done, otherwise we write it as the minimal union of aligned intervals 
(take out the largest aligned interval in $I$ and repeat).  There are three possibilities:-
\begin{enumerate}
  \item $I = I_1 \cup I_2$ is a union of two intervals of size $x/2$ each (eg. the interval $[T/4 + 1,\ 3T/4]$) \label{first}
  \item $I = I_1 \cup I_2 \cup \ldots \cup I_k$, where each $I_j$ is of a different size.  Note that all interval sizes on the right are powers of $2$ and strictly less than $x$ \label{second}
  \item $I = J \cup J'$ where each $J$ can be written as a union of intervals as in \ref{first} or \ref{second} above
\end{enumerate}

In the first case, 

$$ | \E[h_I] | \leq |\E[h_{I_1}] | + | \E[h_{I_1}] | \leq 2 \cdot \epsilon \cdot \sqrt{x/2} = \sqrt{2} \cdot \epsilon \cdot \sqrt{x} $$

In the second case,
$$ | \E[h_I] | \leq \sum_{j = 1}^k |\E[h_{I_j}] | \leq \epsilon \cdot \sqrt{x} \cdot \sum_{j=1}^\infty \sqrt{1/2^j} =  \frac{1}{\sqrt{2} - 1}  \cdot \epsilon \cdot \sqrt{x} $$

In the third case, 

$$ | \E[h_I] | \leq |\E[h_J]| + |\E[h_J']| \leq \frac{1}{\sqrt{2} - 1} \cdot \epsilon \cdot \sqrt{|J|} \ + $$
$$ \frac{1}{\sqrt{2} - 1} \cdot \epsilon \cdot \sqrt{|J'|} \leq \frac{\sqrt{2}}{\sqrt{2} - 1} \cdot \epsilon \cdot \sqrt{|I|} $$
\end{Proof}

\section{Fractal nature of deterministic inverting sequences}

We will argue that the optimal sequence with height $h$ and inversion ratio $\alpha$ is achieved by the following fractal-like recursive process. 
To construct a sequence of height $h$, recursively generate a sequence $s_1$ of height $(1+\alpha/2) \cdot h$ and $s_2$ of height $\alpha \cdot  h$ respectively. 
Concatenate $s_1$, an inverted copy of $s_2$ followed by another copy of $s_1$. For simplicity for explanation we will ignore rounding errors from the discretization.

It turns out that for large $h$, the ratio of lengths of $s_1$ and $s_2$ is fixed to $\left(\frac{1 + \alpha}{2}\right)^{1/\theta}: \alpha^{1/\theta}$ 
where $\theta$ is a constant defined below.

\begin{claim}
The above process produces an $\alpha$-inverting sequence for $\alpha$ smaller than some constant.
\end{claim}
\begin{Proof}
 Observe that by recurrence any interval that is contained within $s_1$ or $s_2$ is $\alpha$-inverting. 
 The full interval consisting of the three concatenated strings also has an $\alpha$-inversion; and so are the intervals that span the first 
 two and the last two strings. So we only need to argue about intervals that span parts of multiple of these pieces. Consider for example an interval 
 that spans across some suffix of $s_1$ and some prefix of the inverted copy of $s_2$. Now for small enough $\alpha$, the two parts of the interval have heights 
 of opposite signs. So the $\alpha$-inversion in the piece with the larger absolute height suffices to produce an $\alpha$-inversion in the interval. 
 The same argument can be applied for intervals that span part of the first and the third sequence.
\end{Proof}

\begin{claim} \label{c}

Let $s$ be an $\alpha$-inverting sequence of length $t$ (Definition \ref{inversion-deterministic}), where $\alpha$ is bounded above by a constant.
  Then the highest deviation that can be achieved by $s$ for large $t$ is
$t^\theta$ where $\theta$ is the solution to the equation $1 =2((1+\alpha)/2)^{1/\theta} + (\alpha)^{1/\theta}$.  Furthermore, this deviation is actually achieved by the above process.
\end{claim}
\begin{Proof} [Sketch]
We will compute the amount of time $t(h)$ when the process described above first achieves a height $h > 0$. 
By the construction, $t(h)$ satisfies the recurrence $t(h) = 2t \cdot ((1+\alpha) h/2) + t(\alpha h)$.
In the limit, if this recurrence has a solution of the form $h^{1/\theta}$ then note that 
$h^{1/\theta} =  2((1+\alpha) h/2)^{1/\theta} + (\alpha  h)^{1/\theta}$ which means that  
$1 =2((1+\alpha)/2)^{1/\theta} + (\alpha)^{1/\theta}$. The proof can be formalized by sandwiching the solution to the recurrence in the limit 
between the functions  $h^{1/\theta_1}$ and  $h^{1/\theta_2}$ where $\theta_1$ and $\theta_2$ approach $\theta$ from above and below.

To prove the lower bound,  let $t(h)$ denote the required time to produce a height of absolute value $h$ for any $\alpha$-inverting sequence. 
We will prove that for large $h$, $t(h)$ approaches $h^{1/\theta}$.  We know that for large enough $t$ there must be an inversion with ratio $\alpha$. 
So  to achieve height $h$ in time $t$ there must be a sub-interval with height less than  $-\alpha h$. So $t$   can be broken into three segments  of 
lengths $t_1$, $t_2$, $t_3$  with heights $h_1$, $h_2$, $h_3$ such that $h = h_1 + h_2 + h_3$ where
$h_2 \leq -\alpha h$. 
We wish to minimize $t(h) = t_1 + t_2 + t_3 \geq t(h_1) + t(h_2) + t(h_3)$. Since $t(h)$ is non-decreasing in $h$, we may set 
$h_2 =   -\alpha h$ and $h_1 + h_3 = h - h_2  = (1+\alpha) h$ giving
$t(h) = \min \ t(h_1) + t(h_2) + t(\alpha h)$ where $h_1 + h_3 = (1+\alpha) h$.

Note that if $t(h)$ is of the form $h^{1/\theta}$ then it  is convex and so $t(h_1) + t(h_2)$ is minimized when $h_1 = h_3 = \frac{(1+\alpha) \cdot h}{2}$ giving 
$t(h) = 2 t(\frac{(1+\alpha) \cdot h}{2}) + t(\alpha h)$ whose solution approaches $h^{1/\theta}$ in the limit. That the solution must approach $h^{1/\theta}$,
 by looking at the behavior of $\log_t h$ in the limit and sandwiching it between $\theta_1$ and $\theta_2$ that approach $\theta$ from above and below.
 \end{Proof}

\section{Miscellaneous Observations}

\begin{observation} \label{observationrandomwalk}
A uniform random sequence is $(\alpha, q)$ inverting (Definition \ref{inverting-dist}) for some constants $\alpha, q$.  
 Further the probability parameter $q$ can be made as high as  $1 - \eps$  by reducing the inversion ratio $\alpha$ to $\Theta(1/\log(1/\eps))$.
\end{observation}
\begin{Proof}
Let us divide the interval of length $x$ into two halves of length $x/2$ each. With probability $1/2$ the two parts have opposite heights and with constant 
probability both heights have magnitude $\Theta(\sqrt x)$. Thus it has an $\alpha$-inversion with some constant probability for some constant $\alpha$. 
The higher probability statement is obtained similarly by dividing it into $\log(1/\eps)$ intervals of equal length. 
\end{Proof}

\begin{observation} \label{entropy-predictable}
If a string is sampled from the highest entropy distribution with deviation $k \sqrt T$, then it is possible to get an
expected payoff of $\Omega(1) \cdot k \sqrt {T}$ for $k = \Omega(1)$.
\end{observation}
\begin{Proof} [Sketch]
The algorithm simply predicts the sign of $h(s)$ where $s$ is the sequence seen in the first half i.e. $|s| = T/2$.  A simple computation proves the observation.
\end{Proof}

The following theorem shows that the FBM with $H = 1/2+ \delta$ is not $O(\delta)$-unpredictable. In fact, an algorithm can get an expected payoff of 
$\Theta(x^{H})$ on an interval of size $x$ by predicting the sign of the height of the preceding interval of length $x$. 
(It can also be shown that one cannot do better than this if one is only allowed to use the sign of the height of some preceding interval.)

\begin{claim} \label{FBM-not-delta}
The algorithm that predicts an interval of length $x$ using the sign of the height of the preceding interval of length $x$ gets an expected 
payoff of $\Theta(x^{H})$ where the expectation is taken over all values in the preceding interval. Further it is optimal to use a preceding interval of 
length $x$ if one is using the sign of its height.
\end{claim}
\begin{proof}
$E[B_H(sx)|B_H(x)]/B_H(x) = (1/2)(s^{2H}+1-|s-1|^{2H})$ (See \cite{mandelbrot1968fractional}, Section 5.3)

Let us compute the expected payoff if one uses the height of the preceding interval of length $x$ to predict the following interval of length $x$.

$E[B_H((s+1)x)|B_H(sx)]/B_H(sx) = (1/2)((1+1/s)^{2H}+1-(1/s)^{2H})$.
$E[B_H((s+1)x) - B_H(sx)| B_H(sx)] = (1/2)((1+1/s)^{2H}-1-(1/s)^{2H}) B_H(sx)$.
So by predicting the sign of $B_H(sx)$ to predict the following interval of length $x$ the expected payoff is 
$E[sign(B_H(sx)) B_H(sx)] = (1/2)((1+1/s)^{2H}-1-(1/s)^{2H}) E[|B_H(sx)|] = \Theta{(sx)^{H}) (1/2)((1+1/s)^{2H}-1-(1/s)^{2H}}$.

Note that for $s=1$, this is $\Theta(x^{H})$. Further this is the best possible value of the above expression.
\end{proof}

\begin{observation} \label{d}
With continuous prediction the FBM and its binary (discretized) variants have a payoff of $\Omega(\delta T)$
\end{observation}

\begin{Proof}
Observe that if we take a sequence of length $2$ the second bit is correlated to the first by $\Theta(\delta)$. This is true of every even bit. 
The observation follows for the binary variants. For the true FBM the statement holds since if $B_1$ and $B_2$ are the heights  in two adjacent unit 
intervals of the FBM process with hurst coefficient $H = 1/2 + \Theta(\delta)$ then

$E[B_1 + B_2 \ | \ B_1]/B_1 = (1/2)(2^{2H} +1 - 1^{2H}) = 2^{ \Theta(\delta)}$ (See \cite{mandelbrot1968fractional}, Section $5.3$)

Therefore $E[B_2 \ | \ B_1] = (2^{ \Theta(\delta)} -1) \cdot B_1 = \Theta(\delta) \cdot B_1$ for $\delta \le 1$. So again by predicting the sign of $B_1$ one can get a payoff of 
$\Theta(\delta) \cdot B_1 \cdot sign(B_1) = \Theta(\delta) \cdot |B_1|$. This in expectation is $\Theta(\delta)$ as $B_1$ is normally distributed with constant variance.
\end{Proof}

\begin{claim} \label{e}
For any random variable $X$ that only takes non negative values and $E[X^2] = O({(E[X])}^2$, $Pr[X \ge \Omega(E[X]) = \Omega(1)$
\end{claim}
\begin{Proof}
Let $\mu = E[X]$. The the standard deviation $\sigma  = O(\mu) = c\mu$ (say) where $c$ is at most some constant.
We will bound $E[X|X \ge \mu + r c \mu]$ for any $r \in \mathbb{N}$. Note that $Pr[X \ge \mu + r c \mu] \le 1/r^2$.

So  $E[X|X \ge \mu + r c\mu] \le (\mu + r c \mu) + \mu \sum_{i > r} 1/i^2 \le (\mu + r c \mu) + c\mu/r$.

Now $\mu = E[X] = Pr[X < \mu + r c\mu]E[X|X < \mu + r c\mu] +  Pr[X\ge \mu + r c\mu]E[X|X \ge \mu + r c\mu] \le (1-1/r^2) E[X|X < \mu + r c\mu]  + (1/r^2) (\mu + rc\mu+ c\mu/r)$.

By setting $r$ to be a constant that is at least some large multiple of $c$, we can conclude that 
$E[X|X < \mu + r c\mu] = \Omega(\mu)$. So this conditioned random variable $X$ has maximum value and mean value that are the same upto constant factors. 
Thus it must exceed $\Omega(\mu)$ with constant probability. So the unconditioned random variable $X$ must also exceed $\Omega(\mu)$ with a smaller constant probability.
\end{Proof}

\section{Basic tools}

\begin{theorem} [Hoeffding's bound] \cite{hoeffding1963probability} \label{hoeffding}

Let $X_1, X_2, \ldots, X_n$ be independent random variables such that $\E[X_i] = 0$ and $\Pr[X_i \in [a_i, b_i]] = 1$.  Let $S := \sum_i X_i$.  Then,

$$ \Pr[|S| \geq y] \leq 2 \cdot \exp\left( - \frac{2 y^2}{\sum_{i=1}^n (b_i - a_i)^2} \right) $$
\end{theorem}

\begin{theorem} [Berry-Esseen Theorem] \label{berry-esseen} \cite{berry1941accuracy}

Let $X_1, X_2, \ldots, X_n$ be independent random variables such that $\E[X_i] = 0$, $\E[X_i^2] = \sigma_i^2 > 0$, and $\E[|X_i^3|] = \rho_i < \infty$.
Let $\sigma^2 := \sum_i \sigma_i^2$ and $S := \frac{1}{\sigma} \sum_i X_i$.  Then there is an absolute constant $C$ such that

$$ | S - N(0, \sigma^2) | \leq \frac{C}{\sigma} \cdot \max_i \frac{\rho_i}{\sigma_i} $$

Here $|D - D'| := \max_x |\Pr[D \geq x] - \Pr[D' \geq x] |$ denotes the statistical distance between distributions $D$ and $D'$ and $N(\mu, \sigma^2)$ denotes the normal
distribution with mean $\mu$ and variance $\sigma^2$.

\end{theorem}

\end{document}